\documentclass[11pt]{article}

\usepackage[preprint]{acl}

\usepackage{times}
\usepackage{latexsym}
\usepackage[T1]{fontenc}

\usepackage[utf8]{inputenc}

\usepackage{microtype}

\usepackage{inconsolata}

\usepackage{graphicx}
\usepackage{booktabs}
\usepackage{listings}
\usepackage{xcolor}
\usepackage{tcolorbox}
\usepackage{ragged2e}
\usepackage{amsmath, amsthm, amssymb, amsfonts}
\usepackage{tabularx}
\usepackage{hyperref}
\usepackage{cleveref}

\newtheorem{theorem}{Theorem}
\newtheorem{proposition}[theorem]{Proposition}
\newtheorem{corollary}[theorem]{Corollary}
\newtheorem{lemma}[theorem]{Lemma}
\newtheorem{definition}[theorem]{Definition}
\newtheorem{assumption}{Assumption}
\newtheorem{remark}{Remark}

\DeclareMathOperator{\Var}{Var}
\DeclareMathOperator{\Cov}{Cov}
\DeclareMathOperator{\KL}{KL}
\DeclareMathOperator{\Harm}{Harm}
\DeclareMathOperator{\Fail}{Fail}

\newcommand{\E}{\mathbb{E}}

\newcommand{\cH}{\mathcal{H}}
\newcommand{\cR}{\mathcal{R}}
\newcommand{\cV}{\mathcal{V}}

\title{Why Is RLHF Alignment Shallow? A Gradient Analysis}

\author{Robin Young \\
  Department of Computer Science and Technology \\
  University of Cambridge \\
  Cambridge, UK \\
  \texttt{robin.young@cl.cam.ac.uk}}

\begin{document}
\maketitle

\begin{abstract}
Why is safety alignment in LLMs shallow? We prove that gradient-based alignment inherently concentrates on positions where harm is decided and vanishes beyond. Using a martingale decomposition of sequence-level harm, we derive an exact characterization of alignment gradients. The gradient at position $t$ equals the covariance between the conditional expected harm and the score function. This implies that positions beyond the harm horizon where the output's harmfulness is already determined receive zero gradient signal during training. This explains empirical observations that KL divergence between aligned and base models concentrates on early tokens. Consequently, standard alignment objectives cannot produce deep alignment, regardless of optimization quality. We introduce the concept of harm information $I_t$, which quantifies each position's influence on harm, and prove that equilibrium KL divergence tracks this quantity. Finally, we derive an objective based on recovery penalties that creates gradient signal at all positions, providing theoretical grounding for empirically successful data augmentation techniques.
\end{abstract}

\section{Introduction}
\label{sec:intro}

Large language models (LLMs) undergo safety alignment to reduce harmful outputs, typically through reinforcement learning from human feedback (RLHF) or direct preference optimization. A growing body of evidence suggests this alignment is troublingly fragile. It was demonstrated that the behavioral shift from alignment concentrates in the first few tokens; the KL divergence between aligned and base models rapidly decays to near-zero beyond a shallow prefix \cite{qi2025safety}. This shallow alignment creates vulnerabilities to prefilling attacks, where an adversary supplies the first few tokens of a harmful response, bypassing the aligned model's safety behavior.

Why is alignment shallow? The prevailing view treats this as a failure to be remedied perhaps through better training data \citep{qi2025safety} or architectural interventions \citep{zou2024circuitbreaker}. We argue instead that shallow alignment is optimal given standard objectives. The shallowness is not a bug but a necessary consequence of how gradient-based optimization interacts with the structure of sequence-level harm.

We make several contributions in this work. We provide a martingale decomposition of sequence-level harm into per-position innovations, enabling precise analysis of where alignment pressure originates (\Cref{sec:martingale}). We give a characterization of alignment gradients, showing that the gradient at position $t$ equals the covariance between conditional expected harm and the score function (\Cref{thm:gradient}). We prove that positions beyond the harm horizon receive zero gradient, explaining why standard training cannot produce deep alignment (\Cref{thm:zero-gradient}). We introduce the concept of harm information $I_t$, which quantifies each position's influence on the final harm determination, and prove that equilibrium KL divergence tracks this quantity (\Cref{sec:equilibrium}). Finally, we derive a deep alignment objective based on recovery penalties that creates gradient signal at all positions, providing theoretical justification for empirically successful interventions (\Cref{sec:deep}).

\section{Related Work}
\label{sec:related}

Our work connects several lines of research on LLM alignment, adversarial robustness, and the theoretical foundations of RLHF.

Modern LLM alignment traces to RLHF \citep{christiano2017rlhf, ouyang2022training}, which trains a reward model on human preferences and then optimizes the policy via reinforcement learning with a KL penalty against the base model. Direct Preference Optimization \cite{rafailov2023dpo} eliminates the explicit reward model by reparameterizing the RLHF objective, enabling direct policy optimization from preference data. Constitutional AI \citep{bai2022constitutionalaiharmlessnessai} uses AI feedback to scale alignment. These methods all optimize sequence-level objectives with KL regularization, which is the setting we analyze.

\citet{qi2025safety} demonstrated empirically that safety alignment concentrates in early tokens as the KL divergence between aligned and base models decays rapidly beyond a shallow prefix. This creates vulnerabilities to prefilling attacks, where adversarial prefixes bypass safety guardrails \citep{andriushchenko2024jailbreaking}. Prior work formalized alignment depth using Markov chains, proving geometric convergence to harmful absorbing states and showing that ensemble width can compensate for shallow individual models \cite{kao2025safety}. that work characterizes the structure of shallow alignment at inference time but does not explain why training produces it, leaving open whether better optimization, more data, or longer training could yield deep alignment. Our work answers this question in the negative as we prove that zero gradient signal reaches positions beyond the harm horizon under standard objectives, making shallow alignment the optimal solution rather than a training artifact. 

The GCG attack \citep{zou2023universal} finds universal adversarial suffixes via greedy coordinate gradient search, achieving high attack success rates even on black-box models through transfer. Subsequent work has improved attack efficiency \citep{liu2024autodan,liao2024amplegcg} and explored multi-turn jailbreaks \citep{chao2025jailbreaking}. These attacks exploit shallow alignment by steering early tokens toward harmful continuations. Our theoretical framework explains why such attacks succeed: once the harm horizon is bypassed, the model has no training signal to recover.

Representation-level interventions offer an alternative to output-level alignment. Representation engineering \citep{zou2025representation}, refusal direction ablation \citep{arditi2024refusal}, and circuit breakers \citep{zou2024circuitbreaker} operate on internal activations rather than output distributions, and thus fall outside our gradient-theoretic framework.


Recent work analyzes sample complexity \citep{zhu2023principled}, reward model scaling \citep{rafailov2024scaling}, and iterative DPO convergence \citep{xiong2024iterative}; our contribution is orthogonal, characterizing where in the sequence alignment pressure is applied rather than how much data or compute is needed.

\section{Setup}
\label{sec:setup}

\paragraph{Autoregressive Generation.} Let $\cV$ denote the vocabulary. A language model $P_\theta$ defines a distribution over sequences $y = (y_1, \ldots, y_T)$ via the autoregressive factorization:
\begin{equation}
    P_\theta(y) = \prod_{t=1}^{T} P_\theta(y_t \mid y_{<t})
\end{equation}
where $y_{<t} = (y_1, \ldots, y_{t-1})$ and $y_{<1} = \emptyset$ (or a fixed prompt). We write $P_\theta(y_t \mid y_{<t})$ for the conditional distribution at position $t$.

\paragraph{Harm Function.} We assume a harm function $\Harm: \cV^* \to [0,1]$ that maps complete sequences to a harm score. This abstracts reward models, classifiers, or human judgments. We assume $\Harm$ depends only on the output sequence (not on internal model states).

\paragraph{Alignment Objective.} Standard alignment minimizes expected harm subject to a KL penalty against a base model $P_{\text{base}}$:
\begin{equation}
    \label{eq:alignment-objective}
    \cH(\theta) = \lambda \cdot \E_{y \sim P_\theta}[\Harm(y)] + D_{\KL}(P_\theta \| P_{\text{base}})
\end{equation}
where $\lambda > 0$ controls the strength of harm aversion. The KL term prevents collapse to degenerate distributions and preserves capabilities.

\paragraph{Goal.} We seek to understand how gradient-based minimization of Eq~\eqref{eq:alignment-objective} distributes behavioral changes across positions $t = 1, \ldots, T$.

\section{Martingale Decomposition of Harm}
\label{sec:martingale}

The key insight is that expected harm, viewed as a function of the partial sequence $y_{\leq t}$, forms a martingale. This decomposition reveals which positions carry information about harm.

\begin{definition}[Conditional Expected Harm]
For a partial sequence $y_{\leq t} = (y_1, \ldots, y_t)$, define
\begin{equation}
    h_t(y_{\leq t}) := \E[\Harm(y) \mid y_1, \ldots, y_t]
\end{equation}
where the expectation is over continuations $y_{>t}$ drawn from $P_\theta(\cdot \mid y_{\leq t})$.
\end{definition}

\begin{proposition}[Martingale Property]
\label{prop:martingale}
The sequence $(h_t)_{t=0}^T$ is a martingale with respect to the filtration generated by $(y_t)$:
\begin{equation}
    \E[h_t \mid y_{<t}] = h_{t-1}(y_{<t})
\end{equation}
\end{proposition}

\begin{proof}
\small
By the tower property of conditional expectation:
\begin{align}
    \E[h_t(y_{\leq t}) \mid y_{<t}] 
    &= \E[\E[\Harm(y) \mid y_{\leq t}] \mid y_{<t}] \\
    &= \E[\Harm(y) \mid y_{<t}] = h_{t-1}(y_{<t})
\end{align}
\end{proof}

\begin{definition}[Harm Innovation]
The \emph{innovation} at position $t$ is:
\begin{equation}
    \Delta_t := h_t(y_{\leq t}) - h_{t-1}(y_{<t})
\end{equation}
This measures the change in expected harm upon observing token $y_t$.
\end{definition}

\begin{proposition}[Doob Decomposition]
\label{prop:doob}
The harm function admits the decomposition:
\begin{equation}
    \Harm(y) = \E[\Harm] + \sum_{t=1}^{T} \Delta_t
\end{equation}
where the innovations $\Delta_t$ are orthogonal: $\E[\Delta_s \cdot \Delta_t] = 0$ for $s \neq t$.
\end{proposition}

\begin{proof}
The equality follows from telescoping: $h_T = \Harm(y)$ and $h_0 = \E[\Harm]$. For orthogonality, assume $s < t$:
\begin{align}
    \E[\Delta_s \cdot \Delta_t] 
    &= \E[\E[\Delta_s \cdot \Delta_t \mid y_{\leq s}]] \\
    &= \E[\Delta_s \cdot \E[\Delta_t \mid y_{\leq s}]] \\
    &= \E[\Delta_s \cdot 0] = 0
\end{align}
where $\E[\Delta_t \mid y_{\leq s}] = 0$ follows from iterated application of the martingale property.
\end{proof}

\begin{definition}[Harm Information]
\label{def:harm-information}
The \emph{harm information} at position $t$ is:
\begin{equation}
    I_t := \E[\Delta_t^2] = \E[\Var(h_t \mid y_{<t})]
\end{equation}
This quantifies how much variance in harm is explained by position $t$.
\end{definition}

\begin{proposition}[Variance Reduction Characterization]
\label{prop:variance-reduction}
The harm information admits the equivalent representation:
\begin{equation}
    I_t = \E\big[\Var(\Harm \mid y_{<t})\big] - \E\big[\Var(\Harm \mid y_{\leq t})\big]
\end{equation}
That is, $I_t$ equals the expected reduction in conditional variance of harm upon observing token $y_t$.
\end{proposition}

\begin{proof}
We use two applications of the law of total variance, $\Var(X) = \E[\Var(X \mid Z)] + \Var(\E[X \mid Z])$.

Applying with $X = \Harm(y)$ and $Z = y_{\leq t}$:
\begin{equation}
    \Var(\Harm) = \E[\Var(\Harm \mid y_{\leq t})] + \Var(h_t) \tag{A}
\end{equation}
Applying with $Z = y_{<t}$:
\begin{equation}
    \Var(\Harm) = \E[\Var(\Harm \mid y_{<t})] + \Var(h_{t-1}) \tag{B}
\end{equation}
Subtracting (A) from (B):
\begin{equation}
    \begin{aligned}
    \Var(h_t) - \Var(h_{t-1}) &= \E[\Var(\Harm \mid y_{<t})] \\
    &\quad - \E[\Var(\Harm \mid y_{\leq t})]
    \end{aligned}
\end{equation}
From the Doob decomposition and orthogonality of innovations: $\Var(h_t) = \sum_{s=1}^{t} I_s$. Thus $\Var(h_t) - \Var(h_{t-1}) = I_t$, completing the proof.
\end{proof}

\begin{corollary}[Variance Decomposition]
\label{cor:variance}
The total variance of harm decomposes as:
\begin{equation}
    \Var(\Harm(y)) = \sum_{t=1}^{T} I_t
\end{equation}
\end{corollary}

\begin{proof}
By telescoping:
\begin{equation}
\small
    \begin{aligned}
    \sum_{t=1}^{T} I_t &= \sum_{t=1}^{T} \Big(\E[\Var(\Harm \mid y_{<t})] - \E[\Var(\Harm \mid y_{\leq t})]\Big) \\
    &= \Var(\Harm \mid \emptyset) - \Var(\Harm \mid y) \\
    &= \Var(\Harm)
    \end{aligned}
\end{equation}
The final equality uses $\Var(\Harm \mid y) = 0$ since $\Harm(y)$ is determined by $y$.
\end{proof}

\begin{remark}[Information-Theoretic Interpretation]
\label{rem:info-theory}
Define the variance-based conditional information $I_{\mathrm{var}}(X; Z \mid W) := \E[\Var(X \mid W)] - \E[\Var(X \mid W, Z)]$. Then $I_t = I_{\mathrm{var}}(\Harm; y_t \mid y_{<t})$, and the variance decomposition becomes the chain rule:
\begin{equation}
    I_{\mathrm{var}}(\Harm; y) = \sum_{t=1}^{T} I_{\mathrm{var}}(\Harm; y_t \mid y_{<t})
\end{equation}
This parallels the mutual information chain rule $I(\Harm; y) = \sum_t I(\Harm; y_t \mid y_{<t})$. For jointly Gaussian random variables, variance-based and entropy-based information are monotonically related, so $I_{\mathrm{var}} = 0 \Leftrightarrow I = 0$. In general, the relationship is more subtle as variance measures second-moment dependence while mutual information captures all statistical dependence.
\end{remark}

The harm information $I_t$ will prove central: it governs both gradient magnitude during training and KL divergence at equilibrium.

\section{Gradient Characterization}
\label{sec:gradient}

We now derive the key result: an exact formula for the gradient of expected harm with respect to the parameters governing position $t$.

\begin{theorem}[Gradient Characterization]
\label{thm:gradient}
The gradient of expected harm with respect to parameters $\theta$ satisfies:
\begin{equation}
    \begin{aligned}
    \nabla_\theta \E[\Harm(y)] &= \sum_{t=1}^{T} \E_{y_{<t}}\bigg[ \Cov_{y_t \mid y_{<t}}\Big( h_t(y_{\leq t}) \\
    &\qquad\qquad\qquad \nabla_\theta \log P_\theta(y_t \mid y_{<t}) \Big) \bigg]
    \end{aligned}
\end{equation}
\end{theorem}

\begin{proof}
We use the policy gradient identity. For any function $f(y)$:
\begin{equation}
    \nabla_\theta \E_{y \sim P_\theta}[f(y)] = \E_{y \sim P_\theta}\left[ f(y) \cdot \nabla_\theta \log P_\theta(y) \right]
\end{equation}
Applying the autoregressive factorization:
\begin{equation}
    \nabla_\theta \log P_\theta(y) = \sum_{t=1}^{T} \nabla_\theta \log P_\theta(y_t \mid y_{<t})
\end{equation}
Substituting $f(y) = \Harm(y) = \E[\Harm] + \sum_{s=1}^{T} \Delta_s$:
\begin{align}
    \nabla_\theta \E[\Harm] 
    &= \E\bigg[ \left( \E[\Harm] + \sum_s \Delta_s \right) \notag \\
    &\qquad \cdot \sum_t \nabla_\theta \log P_\theta(y_t \mid y_{<t}) \bigg]
\end{align}
The constant $\E[\Harm]$ contributes zero (since $\E[\nabla_\theta \log P_\theta(y_t \mid y_{<t})] = 0$).

For the cross terms $\Delta_s \cdot \nabla_\theta \log P_\theta(y_t \mid y_{<t})$, we consider three cases. 

If $s < t$, then $\Delta_s$ is determined by $y_{\leq s}$, so it is constant given $y_{<t}$. Hence:
\begin{equation}
    \begin{aligned}
    &\E[\Delta_s \cdot \nabla_\theta \log P_\theta(y_t \mid y_{<t})] \\
    &\qquad = \E[\Delta_s \cdot \E[\nabla_\theta \log P_\theta(y_t \mid y_{<t}) \mid y_{<t}]] = 0
    \end{aligned}
\end{equation}

If $s > t$, then by the martingale property, $\E[\Delta_s \mid y_{\leq t}] = 0$. Hence:
\begin{equation}
    \begin{aligned}
    &\E[\Delta_s \cdot \nabla_\theta \log P_\theta(y_t \mid y_{<t})] \\
    &\qquad = \E[\E[\Delta_s \mid y_{\leq t}] \cdot \nabla_\theta \log P_\theta(y_t \mid y_{<t})] = 0
    \end{aligned}
\end{equation}

If $s = t$, this term survives. We have:
\begin{align}
    &\E[\Delta_t \cdot \nabla_\theta \log P_\theta(y_t \mid y_{<t})] \notag \\
    &\qquad = \E_{y_{<t}}\left[ \E_{y_t \mid y_{<t}}[\Delta_t \cdot \nabla_\theta \log P_\theta(y_t \mid y_{<t})] \right] \\
    &\qquad = \E_{y_{<t}}\left[ \Cov_{y_t \mid y_{<t}}(\Delta_t, \nabla_\theta \log P_\theta(y_t \mid y_{<t})) \right]
\end{align}
where the covariance appears because $\E[\nabla_\theta \log P_\theta(y_t \mid y_{<t}) \mid y_{<t}] = 0$.

Finally, note that $\Delta_t = h_t(y_{\leq t}) - h_{t-1}(y_{<t})$, where $h_{t-1}(y_{<t})$ is constant given $y_{<t}$. Thus $\Cov(\Delta_t, \cdot) = \Cov(h_t, \cdot)$, yielding the result.

\end{proof}

\begin{remark}[Intuition]
The gradient at position $t$ depends only on how the choice of $y_t$ correlates with the resulting expected harm $h_t(y_{\leq t})$. If all tokens $y_t$ lead to similar expected harm (given $y_{<t}$) the covariance is small and the gradient is weak.
\end{remark}

\section{The Harm Horizon and Zero-Gradient Theorem}
\label{sec:horizon}

We now formalize when harm is ``determined'' by an early prefix, and prove that later positions receive no gradient signal. The following theorem establishes a tight equivalence between the harm horizon and the decay of harm information.

\begin{theorem}[Harm Horizon Characterization]
\label{thm:horizon-equiv}
The following are equivalent:
\begin{enumerate}
    \item[(i)] $I_t = 0$ for all $t > k$.
    \item[(ii)] There exists a measurable function $g: \cV^k \to [0,1]$ such that $\Harm(y) = g(y_{\leq k})$ almost surely.
\end{enumerate}
We call the smallest such $k$ the \emph{harm horizon}.
\end{theorem}

\begin{proof}
\textbf{(ii) $\Rightarrow$ (i):} 
Assume $\Harm(y) = g(y_{\leq k})$ for some measurable $g$. Let $t > k$. Then
\begin{equation}
\small
    h_t(y_{\leq t}) = \E[\Harm(y) \mid y_{\leq t}] = \E[g(y_{\leq k}) \mid y_{\leq t}] = g(y_{\leq k})
\end{equation}
since the prefix $y_{\leq k}$ is determined by $y_{\leq t}$ when $t > k$. Similarly, $h_{t-1}(y_{<t}) = g(y_{\leq k})$ since $t - 1 \geq k$. Therefore $\Delta_t = h_t - h_{t-1} = 0$ almost surely, which implies $I_t = \E[\Delta_t^2] = 0$.

\textbf{(i) $\Rightarrow$ (ii):}
Assume $I_t = 0$ for all $t > k$. Since $I_t = \E[\Delta_t^2]$ and $\Delta_t^2 \geq 0$, we have $\Delta_t = 0$ almost surely for all $t > k$.

We proceed by backward induction. For $t = T$: $\Delta_T = 0$ a.s.\ implies $h_T = h_{T-1}$ a.s. Continuing down to $t = k + 1$:
\begin{equation}
    h_T = h_{T-1} = \cdots = h_{k+1} = h_k \quad \text{a.s.}
\end{equation}
Since $h_T(y) = \E[\Harm(y) \mid y] = \Harm(y)$, we have $\Harm(y) = h_k(y_{\leq k})$ almost surely. Setting $g := h_k$ completes the proof.
\end{proof}

\begin{theorem}[Zero Gradient Beyond Horizon]
\label{thm:zero-gradient}
Under either equivalent condition of \Cref{thm:horizon-equiv}, for all $t > k$:
\begin{equation}
\small
    \E_{y_{<t}}\left[\Cov_{y_t \mid y_{<t}}\Big(h_t(y_{\leq t}) \nabla_\theta \log P_\theta(y_t \mid y_{<t})\Big)\right] = 0,
\end{equation}
and hence position $t$ contributes nothing to $\nabla_\theta \E[\Harm(y)]$.
\end{theorem}

\begin{proof}
If $I_t = 0$, then $\Var(h_t \mid y_{<t}) = 0$ almost surely, since $I_t = \E[\Var(h_t \mid y_{<t})]$ and variance is non-negative. A random variable with zero conditional variance is almost surely constant given the conditioning. Thus $h_t(y_{\leq t})$ does not depend on $y_t$ given $y_{<t}$, so the covariance with any function of $y_t$ is zero.
\end{proof}

\begin{remark}[Why Alignment is Shallow]
\Cref{thm:zero-gradient} explains why standard alignment produces shallow models. If harm is typically determined by early tokens (e.g. whether the response begins with a refusal) then:
\begin{enumerate}
    \item Positions beyond the harm horizon receive zero gradient.
    \item Alignment pressure concentrates on early positions.
    \item Later positions remain at the base model distribution.
\end{enumerate}
This is \emph{optimal} behavior given the objective as there is no signal to align later positions.
\end{remark}

\section{Gradient Magnitude Scales with Harm Information}
\label{sec:bounds}

Even when harm is not sharply determined by an early prefix, positions with low harm information receive weak gradient signal. We now make this precise.

\begin{definition}[Conditional Fisher Information]
\label{def:fisher}
The Fisher information at position $t$, conditional on prefix $y_{<t}$, is:
\begin{equation}
    \begin{aligned}
    F_t(y_{<t}; \theta) &:= \E_{y_t \mid y_{<t}}\Big[ \nabla_\theta \log P_\theta(y_t \mid y_{<t}) \\
    &\qquad\qquad\quad \cdot \nabla_\theta \log P_\theta(y_t \mid y_{<t})^\top \Big]
    \end{aligned}
\end{equation}
We write $\bar{F}_t(\theta) := \E_{y_{<t}}[F_t(y_{<t}; \theta)]$ for the marginalized Fisher information.
\end{definition}

\begin{lemma}[Pointwise Cauchy-Schwarz Bound]
\label{lem:pointwise-cs}
For any fixed prefix $y_{<t}$, let $C_t(y_{<t}) := \Cov_{y_t \mid y_{<t}}(h_t, \nabla_\theta \log P_\theta(y_t \mid y_{<t}))$. Then:
\begin{equation}
    \left\| C_t(y_{<t}) \right\|^2 \leq \Var_{y_t \mid y_{<t}}(h_t) \cdot \text{tr}(F_t(y_{<t}; \theta))
\end{equation}
\end{lemma}

\begin{proof}
By Cauchy-Schwarz applied to each coordinate:
\begin{equation}
    \left| \Cov(h_t, s_t^{(j)}) \right|^2 \leq \Var(h_t) \cdot \Var(s_t^{(j)})
\end{equation}
where $s_t^{(j)} = \partial_{\theta_j} \log P_\theta(y_t \mid y_{<t})$. Summing over coordinates $j$:
\begin{equation}
    \begin{aligned}
    \sum_j \left| \Cov(h_t, s_t^{(j)}) \right|^2 &\leq \Var(h_t) \cdot \sum_j \Var(s_t^{(j)}) \\
    &= \Var(h_t) \cdot \text{tr}(F_t(y_{<t}))
    \end{aligned}
\end{equation}
where the last equality uses $\E[s_t^{(j)}] = 0$ (the score has zero mean).
\end{proof}

\begin{theorem}[Gradient Bound]
\label{thm:cs-bound}
The gradient contribution from position $t$ satisfies:
\begin{equation}
    \left\| G_t(\theta) \right\|^2 \leq \E_{y_{<t}}\left[ \Var_{y_t \mid y_{<t}}(h_t) \cdot \text{tr}(F_t(y_{<t}; \theta)) \right]
\end{equation}
where $G_t(\theta) := \E_{y_{<t}}[C_t(y_{<t})]$ is the gradient contribution defined in \Cref{thm:gradient}.

In particular, if the Fisher information is uniformly bounded by $\text{tr}(F_t(y_{<t})) \leq \bar{F}$ for all prefixes, then:
\begin{equation}
    \left\| G_t(\theta) \right\|^2 \leq I_t \cdot \bar{F}
\end{equation}
\end{theorem}

\begin{proof}
By Jensen's inequality (since $\|\cdot\|^2$ is convex):
\begin{equation}
    \|G_t\|^2 = \left\| \E_{y_{<t}}[C_t(y_{<t})] \right\|^2 \leq \E_{y_{<t}}\left[ \|C_t(y_{<t})\|^2 \right]
\end{equation}
Applying \Cref{lem:pointwise-cs} to each term in the expectation:
\begin{equation}
    \leq \E_{y_{<t}}\left[ \Var_{y_t \mid y_{<t}}(h_t) \cdot \text{tr}(F_t(y_{<t})) \right]
\end{equation}
For the second claim, if $\text{tr}(F_t(y_{<t})) \leq \bar{F}$ uniformly:
\begin{equation}
    \leq \bar{F} \cdot \E_{y_{<t}}\left[ \Var_{y_t \mid y_{<t}}(h_t) \right] = \bar{F} \cdot I_t
\end{equation}
\end{proof}

\begin{remark}[Qualitative Conclusion]
The key insight is that $\|G_t\| = O(\sqrt{I_t})$: gradient magnitude scales with the square root of harm information. Positions where $I_t \approx 0$ receive negligible gradient signal, regardless of the Fisher information. This provides a smooth interpolation between the sharp case (\Cref{thm:zero-gradient}, where $I_t = 0$ exactly) and the general case.
\end{remark}

\begin{remark}[Tightness]
The pointwise bound (\Cref{lem:pointwise-cs}) is tight when the score is proportional to the harm innovation:
\begin{equation}
    \nabla_\theta \log P_\theta(y_t \mid y_{<t}) \propto h_t(y_{\leq t}) - \E[h_t \mid y_{<t}]
\end{equation}
The bound in \Cref{thm:cs-bound} may be loose due to Jensen's inequality if $C_t(y_{<t})$ varies in direction across different prefixes.
\end{remark}

\section{Equilibrium Analysis}
\label{sec:equilibrium}

We now characterize the equilibrium of the alignment objective \eqref{eq:alignment-objective}.

\begin{theorem}[Equilibrium Characterization]
\label{thm:equilibrium}
Assume the alignment objective $\cH(\theta) = \lambda \E[\Harm(y)] + D_{\KL}(P_\theta \| P_{\text{base}})$ is minimized at $\theta^*$, and that $P_{\theta^*}$ is in an exponential family with the Fisher information matrix $\bar{F}_t$ positive definite at the base model. Then for small $\lambda$, the per-position KL divergence satisfies:
\begin{equation}
    \begin{aligned}
    D_{\KL}^{(t)}(\theta^*) &:= \E_{y_{<t}}\left[ D_{\KL}(P_{\theta^*}(\cdot \mid y_{<t}) \| P_{\text{base}}(\cdot \mid y_{<t})) \right] \\
    &= O(\lambda^2 I_t)
    \end{aligned}
\end{equation}

\end{theorem}

\begin{proof}
At a local minimum, the first-order condition for position $t$ is:
\begin{equation}
    \lambda G_t(\theta^*) + \nabla_\theta D_{\KL}^{(t)}(\theta^*) = 0
\end{equation}
For an exponential family, the KL divergence from the base model satisfies $D_{\KL}^{(t)}(\theta) = \frac{1}{2}(\theta - \theta_{\text{base}})^\top \bar{F}_t (\theta - \theta_{\text{base}}) + O(\|\theta - \theta_{\text{base}}\|^3)$ near $\theta_{\text{base}}$. The gradient is $\nabla_\theta D_{\KL}^{(t)} = \bar{F}_t(\theta - \theta_{\text{base}}) + O(\|\theta - \theta_{\text{base}}\|^2)$.

For small $\lambda$, we have $\theta^* = \theta_{\text{base}} + O(\lambda)$. Substituting into the first-order condition and solving to leading order:
\begin{equation}
    \theta^* - \theta_{\text{base}} = -\lambda \bar{F}_t^{-1} G_t(\theta_{\text{base}}) + O(\lambda^2)
\end{equation}
Thus:
\begin{equation}
    D_{\KL}^{(t)}(\theta^*) = \frac{\lambda^2}{2} G_t^\top \bar{F}_t^{-1} G_t + O(\lambda^3)
\end{equation}
By \Cref{thm:cs-bound}, $\|G_t\|^2 \leq \bar{F} \cdot I_t$ when the Fisher information is bounded, so $D_{\KL}^{(t)} = O(\lambda^2 I_t)$.
\end{proof}

\begin{corollary}[KL Tracks Harm Information]
Positions with $I_t \approx 0$ have $D_{\KL}^{(t)} \approx 0$: they remain close to the base model.
\end{corollary}

This provides theoretical grounding for the empirical observation of \citet{qi2025safety} that the KL divergence profile mirrors the harm information profile.

\section{Deep Alignment via Recovery Penalties}
\label{sec:deep}

We established that standard objectives cannot produce deep alignment when harm has a short horizon. To achieve depth, we modify the objective to create gradient signal at all positions. We introduce the deep alignment objective, derive an exact equilibrium characterization paralleling \Cref{thm:equilibrium}, and prove quantitative robustness guarantees against prefilling attacks.

\begin{definition}[Recovery Tokens and Failure]
Let $\cR \subset \cV$ be a set of recovery tokens tokens whose presence at position $t$ indicates the model is attempting to redirect toward safe behavior (e.g. tokens encoding ``I'', ``Sorry'', ``I cannot''). The recovery failure at position $t$ is $\Fail_t(\theta, y_{<t}) := 1 - P_\theta(y_t \in \cR \mid y_{<t})$.
\end{definition}

\begin{definition}[Deep Alignment Objective]
\label{def:deep-objective}
The deep alignment objective is:
\begin{equation}
    \begin{aligned}
    \cH_{\text{deep}}(\theta) &= \lambda \E_{y \sim P_\theta}[\Harm(y)] \\
    &\quad + \mu \E_{y \sim Q}\left[ \sum_{t=1}^{T} \gamma^{t-1} \Fail_t(\theta, y_{<t}) \right] \\
    &\quad + D_{\KL}(P_\theta \| P_{\text{base}})
    \end{aligned}
\end{equation}
where $Q$ is an adversarial prefix distribution, $\gamma \in (0, 1]$ is a discount factor, and $\mu > 0$ controls recovery penalty strength.
\end{definition}

\begin{definition}[Recovery Information]
\label{def:recovery-info}
The \emph{recovery information} at position $t$ is
\begin{equation}
    J_t(\theta) := \E_{y_{<t} \sim Q}\big[p_t(1 - p_t)\big],
\end{equation}
where $p_t := P_\theta(y_t \in \cR \mid y_{<t})$. This is the analogue of harm information $I_t$ for the recovery penalty: it measures how much the choice of $y_t$ affects recovery.
\end{definition}

The recovery penalty has the same covariance-with-score structure as the harm gradient in \Cref{thm:gradient}. Define the \emph{recovery gradient} $\tilde{G}_t(\theta) := -\E_{y_{<t} \sim Q}[\Cov_{y_t \mid y_{<t}}(\mathbf{1}[y_t \in \cR],\, \nabla_\theta \log P_\theta(y_t \mid y_{<t}))]$. The Cauchy--Schwarz machinery of \Cref{sec:bounds} transfers directly: $\|\tilde{G}_t\|^2 \leq \bar{F} \cdot J_t$ (see Appendix \Cref{app:deep-proofs} for the precise bound).

\subsection{Deep Equilibrium Characterization}

\begin{theorem}[Deep Equilibrium]
\label{thm:deep-equilibrium}
Under the same assumptions as \Cref{thm:equilibrium}, for small $\lambda$ and $\mu$ the minimizer $\theta^*$ of $\cH_{\mathrm{deep}}$ satisfies:
\begin{equation}
    D_{\KL}^{(t)}(\theta^*) = \frac{1}{2}\left\|\lambda G_t + \mu\gamma^{t-1}\tilde{G}_t\right\|_{\bar{F}_t^{-1}}^2 + O\big((\lambda + \mu)^3\big),
    \label{eq:deep-equil}
\end{equation}
where $\|v\|_M^2 := v^\top M v$. In particular:
\begin{enumerate}
    \item[\textup{(i)}] For $t > k$ (beyond the harm horizon), $G_t = 0$ by \Cref{thm:zero-gradient}, so
    \begin{equation}
        D_{\KL}^{(t)}(\theta^*) = \frac{\mu^2 \gamma^{2(t-1)}}{2}\|\tilde{G}_t\|_{\bar{F}_t^{-1}}^2 + O(\mu^3).
    \end{equation}
    This is generically positive wherever $Q$ has support and $J_t > 0$.
    \item[\textup{(ii)}] Setting $\mu = 0$ recovers $D_{\KL}^{(t)} = O(\lambda^2 I_t)$ from \Cref{thm:equilibrium}.
\end{enumerate}
\end{theorem}

\begin{proof}[Proof sketch]
The recovery gradient $\nabla_\theta R_t(\theta) = \tilde{G}_t(\theta)$ follows from the same score-function identity used in \Cref{thm:gradient}, since $Q$ does not depend on $\theta$. With decoupled parameters, the first-order condition at position $t$ becomes $\lambda G_t + \mu\gamma^{t-1}\tilde{G}_t + \bar{F}_t(\theta^* - \theta_{\mathrm{base}}) = 0$ to leading order. Solving and substituting into the quadratic KL expansion yields \eqref{eq:deep-equil}. The full proof is in App. \Cref{app:deep-proofs}.
\end{proof}

The discount factor $\gamma$ controls the depth profile: $\gamma = 1$ produces uniform recovery KL across positions, while $\gamma < 1$ induces geometric decay $\gamma^{2(t-1)}$. The total recovery KL beyond the horizon scales as $O(\mu^2/(1-\gamma^2))$, diverging as $\gamma \to 1$---formalizing the tradeoff between alignment depth and capability preservation.

\subsection{Exact Recovery Probability and Robustness}

The perturbative analysis above requires small $\mu$. For positions beyond the harm horizon, the deep objective admits an exact solution in distribution space.

\begin{theorem}[Exact Recovery Probability]
\label{thm:exact-recovery}
For a fixed prefix $y_{<t}$ with $t > k$, let $p_0 := P_{\mathrm{base}}(y_t \in \cR \mid y_{<t})$ and $\beta := \mu\gamma^{t-1}$. The optimal conditional distribution is the Gibbs measure
\begin{equation}
    P^*(y_t \mid y_{<t}) \propto P_{\mathrm{base}}(y_t \mid y_{<t}) \exp\big(\beta \cdot \mathbf{1}[y_t \in \cR]\big),
\end{equation}
with optimal recovery probability and KL cost:
\begin{align}
    P^*(\cR \mid y_{<t}) &= \sigma\big(\operatorname{logit}(p_0) + \beta\big), \label{eq:sigmoid-shift} \\
    D_{\KL}(P^* \| P_{\mathrm{base}}) &= \beta\, P^*(\cR \mid y_{<t}) - \log Z, \label{eq:kl-cost-main}
\end{align}
where $Z = (1-p_0) + p_0 e^\beta$ and $\sigma(x) = (1+e^{-x})^{-1}$.
\end{theorem}

\begin{proof}[Proof sketch]
Beyond the harm horizon, $G_t = 0$ and the deep objective decouples across prefixes into independent KL-penalized linear problems. Standard variational calculus yields the Gibbs form; the sigmoid and KL expressions follow by direct computation. See App. \Cref{app:deep-proofs}.
\end{proof}

The key insight of \eqref{eq:sigmoid-shift} is that the recovery penalty acts as an additive shift of $\beta$ in the \emph{log-odds} of recovery. This enables a clean robustness guarantee.

\begin{definition}[Recoverability]
\label{def:recoverability}
A model $P_\theta$ is \emph{$(Q,\epsilon)$-recoverable at depth $T$} if $P_\theta(y_t \in \cR \mid y_{<t}) \geq \epsilon$ for all $1 \leq t \leq T$ and $Q$-almost all prefixes $y_{<t}$.
\end{definition}

\begin{theorem}[Robustness Guarantee]
\label{thm:robustness}
Suppose $P_{\mathrm{base}}(y_t \in \cR \mid y_{<t}) \geq p_{\min} > 0$ for all $t$ and $Q$-almost all prefixes. Then the minimizer of $\cH_{\mathrm{deep}}$ (beyond the harm horizon) is $(Q, \epsilon^*)$-recoverable at depth $T$, where
\begin{equation}
    \epsilon^* = \sigma\big(\operatorname{logit}(p_{\min}) + \mu\gamma^{T-1}\big).
    \label{eq:epsilon-star}
\end{equation}
The total KL cost of recovery beyond the horizon satisfies
\begin{align}
\sum_{t=k+1}^{T} D_{\mathrm{KL}}^{(t)}(\theta^*) 
&\leq \sum_{t=k+1}^{T} \mu\gamma^{t-1} - \notag\\
&\sum_{t=k+1}^{T} \log\bigl((1 - p_{\min}) + p_{\min} e^{\mu\gamma^{t-1}}\bigr)
\label{eq:total-kl-bound}
\end{align}
\end{theorem}

\begin{proof}[Proof sketch]
The sigmoid form \eqref{eq:sigmoid-shift} is monotone in $p_0$, so the worst case is $p_0 = p_{\min}$. The worst position is $t = T$ (where $\gamma^{t-1}$ is smallest). The KL bound follows from $P^*(\cR) \leq 1$ and monotonicity of $Z$ in $p_0$. The full proof is in \Cref{app:deep-proofs}.
\end{proof}

\begin{corollary}[Prefilling Attack Resistance]
\label{cor:prefilling}
Under standard alignment, for $t > k$ the model has no additional recovery capability beyond $P_{\mathrm{base}}$. Under deep alignment, recovery probability is at least $\epsilon^* > p_{\min}$ at every position. For $\gamma < 1$, the minimum attack prefix length needed to suppress recovery below $\delta$ is
\begin{equation}
    t_{\mathrm{attack}}(\delta) = 1 + \frac{\log \mu - \log(\operatorname{logit}(\delta) - \operatorname{logit}(p_{\min}))}{\log(1/\gamma)},
\end{equation}
which grows logarithmically in $\mu$. For $\gamma = 1$ and $\mu \geq \operatorname{logit}(\delta) - \operatorname{logit}(p_{\min})$, no finite prefix suffices.
\end{corollary}


\section{Discussion}
\label{sec:discussion}

\subsection{Implications for Alignment Practice}

Our main result is that shallow alignment is not a training failure under this formalization as it is the optimal solution to standard objectives when harm has a short horizon. The gradient characterization (\Cref{thm:gradient}) shows that alignment pressure naturally concentrates where harm is determined. This reframes the problem as that achieving deep alignment requires objective modification, not better optimization or more data.

Deeper alignment requires more total KL divergence from the base model. Our equilibrium characterization shows $D_{\KL}^{(t)} = O(\lambda^2 I_t)$: to maintain alignment at position $t$, we must pay a KL cost proportional to the harm information there. Expanding $I_t$ to later positions (via recovery penalties) thus increases total divergence. This suggests a tradeoff between alignment depth and capability preservation as deep alignment may sacrifice some fluency or knowledge by forcing the model further from its base distribution.

Standard alignment evaluations test whether models refuse harmful prompts at the start of generation. Our analysis suggests this is insufficient. A model may appear aligned under such tests while remaining vulnerable to prefilling attacks. Robust evaluation should measure recovery probability $P(y_t \in \cR \mid y_{<t}^{\text{harmful}})$ at multiple positions, not just initial refusal rates.

\subsection{Connection to Adversarial Robustness}

Our formalization can explain the success of prefilling attacks \citep{andriushchenko2024jailbreaking}. An attacker who supplies tokens covering the harm horizon bypasses all positions with $I_t > 0$. The model then continues from a region where it received no alignment training. This is not a failure of training but a limitation as there was no gradient signal to learn recovery behavior.

An attacker with knowledge of the harm information profile can craft minimal prefilling attacks by supplying just enough tokens to cover the high-$I_t$ region. Our Cauchy-Schwarz bound (\Cref{thm:cs-bound}) suggests that gradient magnitude decays with $\sqrt{I_t}$, so attackers should target the ``cliff edge'' where $I_t$ drops to near-zero.

Our deep alignment objective (\Cref{def:deep-objective}) hypothetically provides a defense by creating $I_t^{\text{eff}} > 0$ at all positions, we ensure the model receives training signal for recovery even from deep in a harmful sequence. This is more robust than input filtering, which can be evaded, or adversarial training, which only addresses seen attacks.

\subsection{Connection to Representation Engineering}

Circuit breakers \citep{zou2024circuitbreaker} and representation engineering \citep{zou2025representation} operate on internal activations rather than output distributions. Our framework applies to output-level alignment (RLHF, DPO) and does not directly characterize these methods. However, we conjecture a connection. If harmful representations activate at specific layers, the ``representation horizon'' may relate to our harm horizon. Investigating this connection is an important direction for future work.

\citet{arditi2024refusal} show that refusal behavior is mediated by a single direction in activation space. This suggests refusal is a ``shallow'' feature that can be manipulated independently of deeper representations. Our recovery penalty encourages refusal-like behavior at all positions, potentially strengthening this direction throughout the model's computation.

\section{Conclusion}
\label{sec:conclusion}

We have provided a gradient-theoretic analysis for why safety alignment in LLMs is shallow under the current popular training methods. The martingale structure of sequence-level harm implies that gradient signal cannot propagate beyond the harm horizon. Standard alignment objectives are therefore incapable of producing deep alignment, regardless of optimization quality. Our deep alignment objective, based on recovery penalties at all positions, offers a principled path to robust safety. We hope this theoretical framework guides future work on alignment techniques that are truly deep.

\section*{Limitations}
\label{sec:limitations}

We now discuss the limitations of our work. These limitations suggest important directions for future work and caveats for interpreting our results.

We analyze gradients with respect to output distributions $P_\theta(y_t | y_{<t})$. Representation-level interventions such as circuit breakers and activation steering operate on internal model states and may achieve alignment depth through mechanisms we do not capture. Formalizing the connection between ``representation depth'' and our ``output depth'' requires a theory of how internal representations map to output distributions.

We treat $\Harm(y)$ as fixed and known. In practice, harm is estimated by a reward model trained on finite preference data, which may have its own harm horizon as an artifact of training, may produce spurious gradient signal through estimation errors, and may be unreliable on adversarial inputs due to distribution shift. Our theory characterizes alignment to the reward model's notion of harm, which may diverge from true harm.

Our analysis conditions on a fixed prompt $x$. Different prompts induce different harm information profiles $I_t(x)$, and a model may be deeply aligned for training prompts but shallowly aligned for novel ones. We provide no aggregation principle for combining $I_t(x)$ across prompts into a single measure of alignment depth.

The deep alignment objective assumes recovery is desirable at all positions, but mid-sequence refusal may sometimes be worse than completing the response and providing context. The feasibility of recovery is formalized by the $p_{\min} > 0$ condition in \Cref{thm:robustness}: when the base model assigns negligible probability to recovery tokens, the required penalty $\mu$ grows as $\log(1/p_{\min})$, which may be impractically large deep into harmful sequences.

Zero gradient beyond the harm horizon characterizes where training pressure is applied, not where behavior changes. The base model may already be safe (or harmful) at later positions independently of alignment. Our robustness guarantee (\Cref{thm:robustness}) partially closes this gap by characterizing recovery probabilities directly, but only for the deep objective. The standard objective remains characterized only through KL.

Our formalism operates at the token level, but harm is a semantic property. The harm horizon in token space may not correspond to the semantic horizon, as a model might commit to harmful content early but realize it through varied token paths. A semantic-level theory would be more fundamental but remains open.

The equilibrium characterization describes the fixed point of optimization, not the path to it. The exact variational solution (\Cref{thm:exact-recovery}) removes the small-$\lambda$ and exponential family assumptions of the perturbative analysis, but assumes unrestricted optimization over per-position distributions, which finite training with early stopping may not achieve.

Transformers have finite capacity and may lack the expressiveness to implement recovery from arbitrary harmful prefixes. The Gibbs solution (\Cref{thm:exact-recovery}) specifies what the optimal conditional distribution looks like, but whether a finite-capacity transformer can approximate it across the diversity of harmful prefixes is an open question.

Finally, our analysis covers single-turn generation. Multi-turn conversations allow recovery across turns. For example, refusing in turn $n+1$ after partially complying in turn $n$ and introduce harm horizons that span multiple exchanges, requiring a theory of inter-turn credit assignment that we do not provide.

\bibliography{custom}

\appendix

\section{Relaxing the Exponential Family Assumption}
\label{app:shared-params}

\Cref{thm:equilibrium} assumes an exponential family parameterization, which provides per-position natural parameters and decouples positions in the first-order optimality condition. Transformers share parameters $\theta$ across all positions. We show that the quadratic KL expansion holds for any smooth parameterization, but that shared parameters introduce cross-position coupling that changes the equilibrium characterization in a theoretically informative way.

\subsection{General Quadratic Expansion}

\begin{lemma}[Fisher--Hessian Identity]
\label{lem:fisher-hessian}
For any smoothly parameterized family $P_\theta(y_t \mid y_{<t})$, the Hessian of the per-position KL divergence at $\theta = \theta_{\mathrm{base}}$ equals the Fisher information:
\begin{equation}
    \frac{\partial^2 D_{\mathrm{KL}}^{(t)}(\theta)}{\partial \theta_i \, \partial \theta_j}\bigg|_{\theta_{\mathrm{base}}} = \bar{F}_t(\theta_{\mathrm{base}})_{ij},
\end{equation}
where $\bar{F}_t(\theta) = \mathbb{E}_{y_{<t}}\!\bigl[F_t(y_{<t};\theta)\bigr]$ is the marginalized Fisher information (\Cref{def:fisher}).
\end{lemma}

\begin{proof}
Write $D_{\mathrm{KL}}^{(t)}(\theta) = \mathbb{E}_{y_{<t}}\!\bigl[\mathrm{KL}(P_\theta(\cdot \mid y_{<t}) \,\|\, P_{\mathrm{base}}(\cdot \mid y_{<t}))\bigr]$. At $\theta = \theta_{\mathrm{base}}$, $D_{\mathrm{KL}}^{(t)} = 0$ and $\nabla_\theta D_{\mathrm{KL}}^{(t)} = 0$ (the base model is a stationary point of the KL from itself). For the Hessian, we use the standard identity. For any two distributions in the same smoothly parameterized family,
\begin{align}
    &\frac{\partial^2}{\partial \theta_i \, \partial \theta_j} 
    \mathrm{KL}(P_\theta \| P_{\theta_0})\bigg|_{\theta = \theta_0} \notag \\
    &\quad = \mathbb{E}_{y \sim P_{\theta_0}}\!\left[
    \frac{\partial \log P_\theta(y)}{\partial \theta_i} \cdot 
    \frac{\partial \log P_\theta(y)}{\partial \theta_j}
    \right]_{\theta = \theta_0}
\end{align}
which is the Fisher information matrix at $\theta_0$. This holds for any smooth parameterization; the exponential family is not required. Marginalizing over $y_{<t}$ gives the result.
\end{proof}

It follows that for any smooth parameterization, the per-position KL admits the expansion
\begin{align}
    D_{\mathrm{KL}}^{(t)}(\theta) 
    &= \tfrac{1}{2}(\theta - \theta_{\mathrm{base}})^\top 
    \bar{F}_t(\theta_{\mathrm{base}})
    (\theta - \theta_{\mathrm{base}}) \notag \\
    &\quad + O(\|\theta - \theta_{\mathrm{base}}\|^3).
    \label{eq:kl-expansion}
\end{align}
The exponential family assumption in \Cref{thm:equilibrium} can therefore be replaced by:

\begin{assumption}[Smooth Parameterization]
\label{assump:smooth}
$P_\theta(y_t \mid y_{<t})$ is twice continuously differentiable in $\theta$ for all $y_t, y_{<t}$, and $\bar{F}_t(\theta_{\mathrm{base}})$ is positive definite.
\end{assumption}

This is satisfied by transformers with softmax output layers.

\subsection{Shared-Parameter Coupling}

With shared parameters, the alignment objective is
\[
    \mathcal{H}(\theta) = \lambda \, \mathbb{E}[\mathrm{Harm}(y)] + \sum_{t=1}^T D_{\mathrm{KL}}^{(t)}(\theta),
\]
and the first-order condition is a single equation over all positions:
\begin{equation}
    \lambda \sum_{t=1}^T G_t(\theta^*) + \sum_{t=1}^T \nabla_\theta D_{\mathrm{KL}}^{(t)}(\theta^*) = 0.
    \label{eq:shared-foc}
\end{equation}

Substituting the quadratic expansion \eqref{eq:kl-expansion}:

\begin{proposition}[Shared-Parameter Equilibrium]
\label{prop:shared-equil}
Under \Cref{assump:smooth}, for small $\lambda$ the optimal parameter shift is
\begin{equation}
    \theta^* - \theta_{\mathrm{base}} = -\lambda \, \bar{F}_{\mathrm{tot}}^{-1} \sum_{t=1}^T G_t(\theta_{\mathrm{base}}) + O(\lambda^2),
    \label{eq:shared-shift}
\end{equation}
where $\bar{F}_{\mathrm{tot}} = \sum_{t=1}^T \bar{F}_t(\theta_{\mathrm{base}})$ is the total Fisher information.
\end{proposition}

\begin{proof}
Expanding $\nabla_\theta D_{\mathrm{KL}}^{(t)}(\theta) = \bar{F}_t(\theta_{\mathrm{base}})(\theta - \theta_{\mathrm{base}}) + O(\|\theta - \theta_{\mathrm{base}}\|^2)$ and substituting into \eqref{eq:shared-foc}:
\[
    \lambda \sum_t G_t + \bar{F}_{\mathrm{tot}}(\theta^* - \theta_{\mathrm{base}}) + O(\|\theta^* - \theta_{\mathrm{base}}\|^2) = 0.
\]
Since $\theta^* = \theta_{\mathrm{base}} + O(\lambda)$, the higher-order terms are $O(\lambda^2)$, and solving to leading order gives \eqref{eq:shared-shift}.
\end{proof}

The per-position KL divergence at equilibrium is then:
\begin{align}
    D_{\mathrm{KL}}^{(t)}(\theta^*) 
    &= \frac{\lambda^2}{2} \left(\sum_s G_s\right)^\top 
    \bar{F}_{\mathrm{tot}}^{-1} \, \bar{F}_t \, 
    \bar{F}_{\mathrm{tot}}^{-1} 
    \left(\sum_s G_s\right) \notag \\
    &\quad + O(\lambda^3).
    \label{eq:coupled-kl}
\end{align}

\subsection{Comparison to the Decoupled Case}

In the exponential family setting, per-position natural parameters decouple the first-order condition into $T$ independent equations, each yielding $D_{\mathrm{KL}}^{(t)} = O(\lambda^2 \mathcal{I}_t)$. The shared-parameter expression \eqref{eq:coupled-kl} differs in two ways.

\paragraph{Cross-position dependence.} Even if $\mathcal{I}_t = 0$ (so $G_t = 0$), the per-position KL $D_{\mathrm{KL}}^{(t)}(\theta^*)$ can be nonzero. The parameter shift $\Delta\theta = \theta^* - \theta_{\mathrm{base}}$ is driven by gradients from all positions; through shared weights, it perturbs the output distribution at every position, including those beyond the harm horizon.

\paragraph{Incidental vs.\ functional change.} We can decompose the per-position KL into two components:
\begin{align}
    D_{\mathrm{KL}}^{(t)}(\theta^*) 
    &= \underbrace{\frac{\lambda^2}{2} 
    G_t^\top \bar{F}_{\mathrm{tot}}^{-1} 
    \bar{F}_t \, \bar{F}_{\mathrm{tot}}^{-1} G_t
    }_{\text{functional: driven by local } \mathcal{I}_t} +
    \notag \\
    & \underbrace{\frac{\lambda^2}{2}\sum_{s \neq t} 
    G_s^\top \bar{F}_{\mathrm{tot}}^{-1} \bar{F}_t \, 
    \bar{F}_{\mathrm{tot}}^{-1} G_s 
    + \text{cross terms}
    }_{\text{incidental: driven by remote } \mathcal{I}_s}
    \label{eq:decomp}
\end{align}
The functional component vanishes when $\mathcal{I}_t = 0$ (by \Cref{thm:gradient}). The incidental component does not; it reflects changes to the position-$t$ output distribution caused by parameter shifts optimizing other positions.

\begin{proposition}[Incidental Change is Safety-Irrelevant]
\label{prop:incidental}
Let $t > k$ where $k$ is the harm horizon (\Cref{thm:horizon-equiv}). The incidental change at position $t$ satisfies:
\begin{align}
    & \mathbb{E}_{y_{<t}}\!\Bigl[
      \mathrm{Cov}_{y_t \mid y_{<t}}\!\bigl(
        h_t(y_{\leq t}),
        \notag \\
    & \qquad
        \log P_{\theta^*}(y_t \mid y_{<t})
        \notag \\
    & \qquad
        - \log P_{\mathrm{base}}(y_t \mid y_{<t})
      \bigr)
    \Bigr]
    \notag \\
    & \quad = 0.
    \label{eq:safety-irrelevant}
\end{align}
That is, the distributional change at position $t$ is uncorrelated with expected harm.
\end{proposition}

\begin{proof}
Since $t > k$, harm is determined by $y_{\leq k}$ (\Cref{thm:horizon-equiv}), so $h_t(y_{\leq t}) = h_k(y_{\leq k})$ almost surely. This is constant given $y_{<t}$ (since $t > k$ implies $y_{\leq k} \subseteq y_{<t}$). A random variable that is constant given the conditioning has zero covariance with any function of $y_t$.
\end{proof}

\subsection{Refined Equilibrium Characterization}

Combining these observations, we obtain:

\begin{theorem}[Shared-Parameter Equilibrium]
\label{thm:shared-equil}
Under \Cref{assump:smooth}, for small $\lambda$:
\begin{enumerate}
    \item For all positions $t$: \, $D_{\mathrm{KL}}^{(t)}(\theta^*) = O\!\left(\lambda^2 \left\|\bar{F}_{\mathrm{tot}}^{-1}\sum_s G_s\right\|_{\bar{F}_t}^2\right)$.
    \item For positions $t$ beyond the harm horizon $D_{\mathrm{KL}}^{(t)}(\theta^*)$ may be nonzero, but the distributional change is uncorrelated with harm (\Cref{prop:incidental}).
    \item If per-position parameters decouple (as in the exponential family case), this reduces to $D_{\mathrm{KL}}^{(t)}(\theta^*) = O(\lambda^2 \mathcal{I}_t)$, recovering \Cref{thm:equilibrium}.
\end{enumerate}
\end{theorem}

The qualitative conclusion of the main body is preserved but refined. Standard alignment changes late-position behavior incidentally, through the coupling of shared parameters, but these changes carry no safety-relevant information. Deep alignment, namely the distributional change that is functionally related to harm, remains confined to the harm horizon.

The shared-parameter analysis generates a prediction that distinguishes it from the decoupled theory. For aligned models with shared parameters, positions $t$ beyond the harm horizon satisfy $D_{\mathrm{KL}}^{(t)}(\theta^*) > 0$ but $P_{\theta^*}(y_t \in R \mid y_{<t}^{\mathrm{harmful}}) \approx P_{\mathrm{base}}(y_t \in R \mid y_{<t}^{\mathrm{harmful}})$. That is, the aligned model's output distribution differs from the base model at late positions, but this difference does not manifest as increased recovery capability.

The decoupled theory predicts $D_{\mathrm{KL}}^{(t)} \approx 0$ for $t > k$. The shared-parameter theory predicts nonzero KL that is safety-irrelevant. Measuring both $D_{\mathrm{KL}}^{(t)}$ and recovery probability $P(y_t \in R \mid y_{<t}^{\mathrm{harmful}})$ across positions would distinguish the two accounts and, more practically, would test whether nonzero late-position KL should be interpreted as evidence of deep alignment (it should not).

\section{Proofs and Supplementary Results for Deep Alignment}
\label{app:deep-proofs}

This appendix provides full proofs for the results in \Cref{sec:deep} and additional remarks on their scope.

\subsection{Recovery Gradient Bound}

\begin{lemma}[Recovery Gradient Bound]
\label{lem:recovery-bound}
The recovery gradient $\tilde{G}_t(\theta) := -\E_{y_{<t} \sim Q}[\Cov_{y_t \mid y_{<t}}(\mathbf{1}[y_t \in \cR],\, \nabla_\theta \log P_\theta(y_t \mid y_{<t}))]$ satisfies
\begin{equation}
    \|\tilde{G}_t(\theta)\|^2 \leq \E_{y_{<t} \sim Q}\big[\mathrm{tr}(F_t(y_{<t};\theta))\big] \cdot J_t(\theta).
\end{equation}
If $\mathrm{tr}(F_t(y_{<t};\theta)) \leq \bar{F}$ uniformly, then $\|\tilde{G}_t\|^2 \leq \bar{F} \cdot J_t$.
\end{lemma}

\begin{proof}
The argument is identical in structure to the proof of \Cref{thm:cs-bound}. Write $s_t^{(j)} = \partial_{\theta_j} \log P_\theta(y_t \mid y_{<t})$. For a fixed prefix $y_{<t}$, by Cauchy--Schwarz applied coordinate-wise:
\begin{align}
&\left|\mathrm{Cov}_{y_t|y_{<t}}\!\left(\mathbf{1}[y_t \in \mathcal{R}],\, s_t^{(j)}\right)\right|^2 \notag\\
&\quad \leq \mathrm{Var}_{y_t|y_{<t}}\!\left(\mathbf{1}[y_t \in \mathcal{R}]\right) \notag\\
&\quad \cdot \mathrm{Var}_{y_t|y_{<t}}\!\left(s_t^{(j)}\right).
\end{align}
Summing over coordinates $j$ and using $\Var(s_t^{(j)}) = \E[(s_t^{(j)})^2]$ (since $\E[s_t^{(j)}] = 0$):
\begin{align}
&\left\|\mathrm{Cov}_{y_t|y_{<t}}\!\left(\mathbf{1}[y_t \in \mathcal{R}],\, \nabla_\theta \log P_\theta\right)\right\|^2 \notag\\
&\quad \leq p_t(1-p_t) \cdot \mathrm{tr}\!\left(F_t(y_{<t};\theta)\right).
\end{align}
where $p_t = P_\theta(y_t \in \cR \mid y_{<t})$. Taking the outer expectation over $y_{<t} \sim Q$ and applying Jensen's inequality ($\|\E[X]\|^2 \leq \E[\|X\|^2]$):
\begin{align}
    \|\tilde{G}_t\|^2 &\leq \E_{y_{<t}\sim Q}\big[p_t(1-p_t) \cdot \mathrm{tr}(F_t(y_{<t};\theta))\big] \\
    &\leq \E_{y_{<t}\sim Q}\big[\mathrm{tr}(F_t(y_{<t};\theta))\big] \cdot J_t(\theta).
\end{align}
The second inequality uses the uniform bound on the Fisher trace.
\end{proof}

\subsection{Proof of \texorpdfstring{\Cref{thm:deep-equilibrium}}{Theorem (Deep Equilibrium)}}

\begin{proof}[Full proof of \Cref{thm:deep-equilibrium}]
The deep alignment objective decomposes as
\begin{align}
\mathcal{H}_{\mathrm{deep}}(\theta) 
&= \lambda\,\mathbb{E}_{P_\theta}[\mathrm{Harm}(y)] \notag\\
&\quad + \mu \sum_{t=1}^T \gamma^{t-1} R_t(\theta) \notag\\
&\quad + \sum_{t=1}^T D_{\mathrm{KL}}^{(t)}(\theta),
\end{align}
where $R_t(\theta) := \mathbb{E}_{y_{<t}\sim Q}[\mathrm{Fail}_t(\theta, y_{<t})]$.

\paragraph{Recovery gradient identification.} Since $Q$ does not depend on $\theta$:
\begin{align}
&\nabla_\theta R_t(\theta) \notag\\
&\quad= -\mathbb{E}_{y_{<t}\sim Q}\!\left[\nabla_\theta P_\theta(y_t \in \mathcal{R} \mid y_{<t})\right] \notag \\
&\quad= -\mathbb{E}_{y_{<t}\sim Q}\!\left[\sum_{r \in \mathcal{R}} P_\theta(r \mid y_{<t})\,\nabla_\theta \log P_\theta(r \mid y_{<t})\right] \notag \\
&\quad= -\mathbb{E}_{y_{<t}\sim Q}\!\Big[\mathbb{E}_{y_t \mid y_{<t}}\!\big[\mathbf{1}[y_t\in\mathcal{R}] \notag\\
&\qquad\qquad \cdot\nabla_\theta\log P_\theta(y_t\mid y_{<t})\big]\Big].
\end{align}
Since $\mathbb{E}_{y_t\mid y_{<t}}[\nabla_\theta\log P_\theta(y_t\mid y_{<t})]=0$, this equals $\tilde{G}_t(\theta)$ from \Cref{def:recovery-info}.

\paragraph{Perturbative expansion.} With decoupled per-position parameters, the first-order condition at position $t$ is
\begin{equation}
    \lambda\, G_t(\theta^*) + \mu\gamma^{t-1}\tilde{G}_t(\theta^*) + \nabla_\theta D_{\KL}^{(t)}(\theta^*) = 0.
    \label{eq:deep-foc-app}
\end{equation}
Using the quadratic KL expansion (\Cref{lem:fisher-hessian}):
\[
    \nabla_\theta D_{\KL}^{(t)}(\theta) = \bar{F}_t(\theta - \theta_{\mathrm{base}}) + O(\|\theta-\theta_{\mathrm{base}}\|^2).
\]
Since $\theta^* = \theta_{\mathrm{base}} + O(\lambda + \mu)$, we evaluate $G_t$ and $\tilde{G}_t$ at $\theta_{\mathrm{base}}$ up to higher-order corrections. Solving \eqref{eq:deep-foc-app} to leading order:
\begin{equation}
    \theta^* - \theta_{\mathrm{base}} = -\bar{F}_t^{-1}\big(\lambda\,G_t + \mu\gamma^{t-1}\tilde{G}_t\big) + O\big((\lambda+\mu)^2\big),
\end{equation}
where $G_t$ and $\tilde{G}_t$ are evaluated at $\theta_{\mathrm{base}}$. Substituting into $D_{\KL}^{(t)}(\theta^*) = \frac{1}{2}(\theta^*-\theta_{\mathrm{base}})^\top\bar{F}_t(\theta^*-\theta_{\mathrm{base}}) + O((\lambda+\mu)^3)$:
\begin{equation}
    D_{\KL}^{(t)}(\theta^*) = \frac{1}{2}\big\|\lambda G_t + \mu\gamma^{t-1}\tilde{G}_t\big\|_{\bar{F}_t^{-1}}^2 + O\big((\lambda+\mu)^3\big).
\end{equation}

\paragraph{Cases.} For $t > k$: $I_t = 0$ implies $G_t = 0$ by \Cref{thm:zero-gradient}, so $D_{\KL}^{(t)} = \frac{\mu^2\gamma^{2(t-1)}}{2}\|\tilde{G}_t\|_{\bar{F}_t^{-1}}^2 + O(\mu^3)$. For $\mu = 0$: the recovery terms vanish, reducing to $\frac{\lambda^2}{2}\|G_t\|_{\bar{F}_t^{-1}}^2 + O(\lambda^3) = O(\lambda^2 I_t)$ by \Cref{lem:recovery-bound,thm:cs-bound}.
\end{proof}

\begin{remark}[Aggregate recovery KL]
\label{rem:aggregate-kl}
For $\gamma < 1$, the total recovery KL beyond the horizon satisfies
\begin{equation}
    \sum_{t>k} D_{\KL}^{(t)}(\theta^*) = O\!\left(\frac{\mu^2}{1-\gamma^2}\right),
\end{equation}
assuming $\|\tilde{G}_t\|_{\bar{F}_t^{-1}}^2$ is bounded uniformly in $t$. This diverges as $\gamma \to 1$, formalizing the depth--capability tradeoff: uniform-depth alignment ($\gamma = 1$) requires unbounded total KL divergence from the base model.
\end{remark}

\begin{remark}[Non-degeneracy of recovery gradient]
\label{rem:nondegen}
The recovery information $J_t > 0$ whenever there exists a prefix $y_{<t}$ in the support of $Q$ with $0 < P_\theta(y_t \in \cR \mid y_{<t}) < 1$. By \Cref{lem:recovery-bound}, this is necessary for $\tilde{G}_t \neq 0$. It is also generically sufficient: $\tilde{G}_t = 0$ despite $J_t > 0$ would require the score function to be orthogonal to the recovery indicator conditional on every prefix, which is a non-generic condition of measure zero in parameter space.
\end{remark}

\subsection{Proof of \texorpdfstring{\Cref{thm:exact-recovery}}{Theorem (Exact Recovery Probability)}}

\begin{proof}[Full proof of \Cref{thm:exact-recovery}]
For $t > k$, the harm gradient $G_t = 0$ and the deep objective's contribution at prefix $y_{<t}$ reduces to the pointwise problem
\begin{equation}
    \min_{P \in \Delta(\cV)}\; \beta\big(1 - P(\cR)\big) + D_{\KL}(P \| P_0),
    \label{eq:pointwise-app}
\end{equation}
where $\beta = \mu\gamma^{t-1}$, $P_0 = P_{\mathrm{base}}(\cdot \mid y_{<t})$, and $\Delta(\cV)$ is the probability simplex.

Introducing a Lagrange multiplier $\nu$ for the constraint $\sum_{y_t} P(y_t) = 1$:
\begin{align}
\frac{\partial}{\partial P(y_t)}\Bigg[
&\sum_{y_t} P(y_t)\log\frac{P(y_t)}{P_0(y_t)} \notag\\
&- \beta\sum_{y_t}P(y_t)\mathbf{1}[y_t\in\mathcal{R}] \notag\\
&- \nu\Big(\sum P - 1\Big)
\Bigg] = 0.
\end{align}
This gives $\log\frac{P^*(y_t)}{P_0(y_t)} + 1 - \beta\,\mathbf{1}[y_t\in\cR] - \nu = 0$, so
\begin{equation}
    P^*(y_t) = P_0(y_t)\exp\big(\beta\,\mathbf{1}[y_t\in\cR] + \nu - 1\big).
\end{equation}
The normalization constraint determines $e^{\nu-1} = 1/Z$ where
\begin{equation}
    Z = \sum_{y_t}P_0(y_t)\exp\big(\beta\,\mathbf{1}[y_t\in\cR]\big) = (1-p_0) + p_0 e^\beta.
\end{equation}
Strict convexity of the KL divergence ensures this is the unique global minimizer.

Recovery probability:
\begin{equation}
    P^*(\cR) = \sum_{r\in\cR}\frac{P_0(r)\,e^\beta}{Z} = \frac{p_0\,e^\beta}{(1-p_0)+p_0\,e^\beta}.
\end{equation}
Dividing numerator and denominator by $p_0\,e^\beta$:
\begin{align}
P^*(\mathcal{R}) 
&= \frac{1}{\dfrac{1-p_0}{p_0}\,e^{-\beta}+1} \notag\\
&= \frac{1}{1+\exp\!\left(-\beta-\log\frac{p_0}{1-p_0}\right)} \notag\\
&= \sigma\!\left(\operatorname{logit}(p_0)+\beta\right).
\end{align}

KL cost:
\begin{align}
D_{\mathrm{KL}}(P^*\|P_0) 
&= \sum_{y_t}P^*(y_t)\log\frac{P^*(y_t)}{P_0(y_t)} \notag\\
&= \sum_{y_t}P^*(y_t)\left[\beta\,\mathbf{1}[y_t\in\mathcal{R}]-\log Z\right] \notag\\
&= \beta\,P^*(\mathcal{R}) - \log Z.
\end{align}
\end{proof}

\begin{remark}[Interpretation]
\label{rem:logit-interpretation}
The recovery penalty acts as an additive shift of $\beta$ in log-odds space. A base probability of $p_0 = 0.01$ is boosted much more in absolute terms than $p_0 = 0.5$ for the same $\beta$, since the sigmoid is steepest near $1/2$. Crucially, the mechanism requires $p_0 > 0$: if the base model assigns zero probability to all recovery tokens for a given prefix, no finite $\beta$ can create recovery.
\end{remark}

\begin{remark}[Importance weighting]
\label{rem:weighting}
The recovery term in $\cH_{\mathrm{deep}}$ samples prefixes from $Q$ while the KL term averages over $P_\theta$. In the non-parametric optimization, the pointwise subproblem at prefix $y_{<t}$ becomes
\begin{align}
\min_{P(\cdot|y_{<t})}\; 
&\frac{q(y_{<t})}{p_\theta(y_{<t})}\,\mu\gamma^{t-1}(1-P(\mathcal{R})) \notag\\
&+ D_{\mathrm{KL}}(P\|P_{\mathrm{base}}(\cdot|y_{<t})),
\end{align}
where the density ratio $q/p_\theta$ acts as an importance weight. The Gibbs solution retains the form of \Cref{thm:exact-recovery} with $\beta$ replaced by $\tilde\beta_t(y_{<t}) := \mu\gamma^{t-1} \cdot q(y_{<t})/p_\theta(y_{<t})$. For harmful prefixes that the aligned model avoids but $Q$ upweights, $\tilde\beta_t \gg \beta_t$, amplifying recovery pressure precisely where it is most needed.
\end{remark}

\subsection{Proof of \texorpdfstring{\Cref{thm:robustness}}{Theorem (Robustness Guarantee)}}

\begin{proof}[Full proof of \Cref{thm:robustness}]

Recovery bound. By \Cref{thm:exact-recovery}, the optimal recovery probability at position $t$ given prefix $y_{<t}$ is $P^*(\cR\mid y_{<t}) = \sigma(\operatorname{logit}(p_0(y_{<t}))+\mu\gamma^{t-1})$. Since $\sigma(\operatorname{logit}(\cdot)+c)$ is monotonically increasing in its argument for any $c > 0$, and $p_0(y_{<t}) \geq p_{\min}$ by assumption:
\begin{align}
P^*(\mathcal{R}\mid y_{<t}) 
&\geq \sigma\!\left(\operatorname{logit}(p_{\min})+\mu\gamma^{t-1}\right) \notag\\
&\geq \sigma\!\left(\operatorname{logit}(p_{\min})+\mu\gamma^{T-1}\right) \notag\\
&= \epsilon^*,
\end{align}
where the second inequality uses $\gamma^{t-1} \geq \gamma^{T-1}$ for $t \leq T$ and $\gamma \in (0,1]$.

KL bound. From \eqref{eq:kl-cost-main}, writing $\beta_t = \mu\gamma^{t-1}$:
\begin{align}
    D_{\KL}^{(t)} &= \E_{y_{<t}\sim Q}\big[\beta_t P^*(\cR\mid y_{<t})-\log Z_t(y_{<t})\big] \\
    &\leq \beta_t \cdot 1 - \log\big((1{-}p_{\min})+p_{\min}\,e^{\beta_t}\big),
\end{align}
where we used $P^*(\cR) \leq 1$ and $Z_t(y_{<t}) = (1{-}p_0)+p_0 e^{\beta_t} \geq (1{-}p_{\min})+p_{\min}e^{\beta_t}$ (since $Z$ is increasing in $p_0$ when $e^{\beta_t} > 1$, as $\partial Z/\partial p_0 = e^{\beta_t} - 1 > 0$). Summing over $t > k$ yields \eqref{eq:total-kl-bound}.
\end{proof}

\begin{proof}[Proof of \Cref{cor:prefilling}]
For $\gamma < 1$, the recovery probability at position $t$ satisfies $\sigma(\operatorname{logit}(p_{\min})+\mu\gamma^{t-1}) \geq \delta$ if and only if $\mu\gamma^{t-1} \geq \operatorname{logit}(\delta)-\operatorname{logit}(p_{\min})$. Taking logarithms: $(t-1)\log\gamma \leq \log(\operatorname{logit}(\delta)-\operatorname{logit}(p_{\min})) - \log\mu$. Since $\log\gamma < 0$, dividing and rearranging:
\begin{equation}
    t \leq 1 + \frac{\log\mu - \log(\operatorname{logit}(\delta)-\operatorname{logit}(p_{\min}))}{\log(1/\gamma)}.
\end{equation}
For $t$ exceeding this threshold, recovery drops below $\delta$. For $\gamma = 1$, the bound $\sigma(\operatorname{logit}(p_{\min})+\mu) \geq \delta$ holds at all positions whenever $\mu \geq \operatorname{logit}(\delta)-\operatorname{logit}(p_{\min})$.
\end{proof}

\subsection{Within-Horizon Recovery}

\begin{proposition}[Within-Horizon Solution]
\label{prop:within-horizon}
For $t \leq k$, the non-parametric solution to $\cH_{\mathrm{deep}}$ takes the form
\begin{align}
P^*(y_t\mid y_{<t}) 
&\propto P_{\mathrm{base}}(y_t\mid y_{<t}) \notag\\
&\cdot\exp\!\Big(-\lambda\,h_t(y_{\leq t}) \notag\\
&\qquad + \mu\gamma^{t-1}\mathbf{1}[y_t\in\mathcal{R}]\Big).
\end{align}
The harm tilt $-\lambda\,h_t$ reduces probability of tokens leading to high expected harm, while the recovery tilt $+\mu\gamma^{t-1}\mathbf{1}[\cdot\in\cR]$ boosts recovery tokens. When these align (recovery tokens have low $h_t$), the effects reinforce; when they conflict (a recovery token would increase expected harm, e.g. by producing an incoherent refusal), the two terms compete, with the tradeoff governed by $\lambda/\mu$.
\end{proposition}

\begin{proof}
For $t \leq k$, the per-prefix subproblem includes both the harm and recovery terms. The Lagrangian for the variational problem has two tilts: $-\lambda h_t(y_{\leq t})$ from the harm objective and $+\mu\gamma^{t-1}\mathbf{1}[y_t \in \cR]$ from the recovery penalty. The KL-penalized optimization with a linear objective $\sum_i c_i P(y_t = i)$ has the well-known Gibbs solution $P^*(y_t) \propto P_0(y_t)\exp(c_{y_t})$. Combining both linear terms gives the result. Note that $h_t(y_{\leq t})$ depends on $y_t$ (unlike the beyond-horizon case), so this is not a simple indicator tilt.
\end{proof}

\subsection{The $p_{\min} > 0$ Assumption}

\begin{remark}[Softmax models]
\label{rem:pmin}
Condition $P_{\mathrm{base}}(y_t \in \cR \mid y_{<t}) \geq p_{\min} > 0$ is automatically satisfied by any model with a softmax output layer and finite logits, since softmax assigns strictly positive probability to every token. However, $p_{\min}$ may be exponentially small for prefixes deep into harmful content. If $p_{\min} = e^{-L}$ for some $L > 0$, achieving $\epsilon^* \geq 1/2$ requires $\mu\gamma^{T-1} \geq L$: the penalty must overcome the base model's log-odds against recovery. For a base model that assigns $p_{\min} = 10^{-6}$ to recovery tokens after a long harmful prefix, this means $\mu\gamma^{T-1} \geq 6\ln 10 \approx 13.8$, which is achievable but requires a substantial penalty that will also incur substantial KL cost.
\end{remark}

\end{document}